\documentclass[twoside,11pt]{article}

\usepackage[abbrvbib, preprint]{jmlr2e}
\usepackage{amsmath}
\usepackage{url}
\usepackage{algorithm2e}
\usepackage{ amssymb }
\usepackage[utf8]{inputenc}
\usepackage{natbib}

\ShortHeadings{}{Measure Theoretic Approach to Nonuniform Learnability}
\firstpageno{1}

\begin{document}
\title{Measure Theoretic Approach to Nonuniform Learnability}

\author{\name Ankit Bandyopadhyay \email ankit.bee17@sot.pdpu.ac.in\\
       \addr Department of Electrical Engineering\\
       Pandit Deendayal Petroleum University\\
       Gandhinagar, Gujarat 382007, India
      }

\maketitle

\begin{abstract}
An earlier introduced characterization of nonuniform learnability that allows the sample size to depend on the hypothesis to which the learner is compared has been redefined using the measure-theoretic approach. Where nonuniform learnability is a strict relaxation of the Probably Approximately Correct (PAC) framework. Introduction of a new algorithm, Generalize Measure Learnability framework (GML), to implement this approach with the study of its sample and computational complexity bounds. Like the Minimum Description Length (MDL) principle, this approach can be regarded as an explication of Occam’s razor. Furthermore, many situations were presented (Hypothesis Classes that are countable where we can apply the GML framework) which we can learn to use the GML scheme and can achieve statistical consistency.
\end{abstract}


\section{Introduction}
Model Selection manages the issue of choosing different competing explanations of data given limited observations. It is one of the most significant problems of inductive and statistical inference. There are commonly two ways to deal with Model selection. The main methodology or the first approach depends on the Structural Risk Minimization (SRM) paradigm (Vapnik and Chervonenkis 1974) and (Vapnik 1995). SRM is beneficial when a learning algorithm depends on a parameter that also depends on the Bias-complexity trade-off. The subsequent methodology depends on the idea of Validation, i.e. Partitioning the training set into two sets where one is utilized for training of candidate models and the second is utilized for choosing which of the given candidate models yields the best outcomes. In Model selection tasks, we try to find the right balance between approximation and estimation errors and to also comment on overfitting or underfitting if our learning algorithm fails to find a predictor with a small risk. \par
Developing further in the first approach, The SRM rule is also advantageous in model selection when prior knowledge is partial. The prior knowledge of a nonuniform learner is weaker, it is attempting to find a model throughout the entire union of hypothesis classes, instead of being focused on one specific hypothesis class. Previous works in this domain include MDL Principle (Rissanen 1978 and (Rissanen 1983) which is based on the following insight: any regularity in the data can be used to compress the data. The more regularities there are, the more the data can be compressed. The relation between SRM and MDL is discussed in (Vapnik 1995). Another method, Complexity Approximation Principle (CAP) which is a generalization of the MDL Principle and MML Principle (Vovk and Gammerman 1999). The definition of nonuniform learnability is connected to the definition given in (Blumer Ehrenfraut Hausslerand Warmuth 1987). These notions are also closely related to the notion of regularization (Tikhonov 1943)\par
In this paper, the notion of nonuniform learnability has been redefined using measure theory. Where we use the idea of Hypothesis class as a measurable space. Moreover, the method in this paper, General Measure Learnability (GML), is an inductive inference that provides a generic solution to the model selection problem. In this setting, we consider Binary classification problems. The elements that are to classified come from a set $\mathcal{X}$, the \emph{domain space}. A classifier or a hypothesis on $\mathcal{X}$ is a function $h \colon \mathcal{X} \to \{0,1\}$. Given a \emph{training sequence} $S$ of labelled examples $(x_i,y_i) \in \mathcal{X} \times \{0,1\}, $ we want to find a hypothesis, that can be used to predict the label of elements from $\mathcal{X}$ not given in $S$. We consider well-known frameworks for this setting in the computational learning theory in the upcoming sections\par
In Section \ref{Section3}, we defined measure among individual $\mathcal{H}_n$, according to Theorem \ref{non}. The class of hypotheses which are not defined according to our scheme are, for example, Class with infinite VC dimension or infinite size class. From Theorem \ref{t7}, we can say that $\mathcal{H}$ is an algebra so we consisdered the smallest $\sigma$-algebra in which we can put a measure on. Now algorithm GML tries to minimize the bound given in Theorem \ref{t9}. We concluded that the sample complexity will be the same as any nonuniform learner. Few examples where even $\mathcal{H}_n$ is NP-hard were mentioned and the fact that hypotheses classes which are NP-hard to learn can be efficiently learnable using some other hypothesis class. Also choosing the optimally based on the given data is NP-hard
(David et. al 2003). Therefore, we tried to restrict our search to a particular index or range of the index. \par
We further discussed the idea of Universal Learning and concluded that Theorem \ref{non} is an essential criterion. The union of computable learners so that model selection becomes computationally efficient will be the future work. The GML follows the criteria of statistically consistent in achieving good rates of convergence whereas, preexisting methods like Maximum Likelihood estimation which do not consider model complexity are inconsistent. Similar methods are also available which might be better than GML but the fundamental similarity being finding a tradeoff between error and complexity.

\section{Preliminaries}
In Statistical Learning theory, a common approach is to assume that the data is generated by some data distribution $\mathcal{D}$ $over$ $\mathcal{X} \to \{0,1\}$. We denote the \emph{true error} of a hypothesis $h$ with respect to the distribution $\mathcal{D}$ by
$$
L_\mathcal{D}(h) = Prob_{(x,y) \sim \mathcal{D}}[h(x) \neq y]
$$
. A learner is typically a function that goes from taking a finite sequence of labelled points $S=((x_1,y_1),....,(x_n,y_n))$ to a hypothesis $h$. Also, we define the \emph{empirical error} of a hypothesis $h$ with respect to a sample $S=((x_1,y_1),....,(x_n,y_n))$ is defined as 
$$
L_S(h) =  \frac{\sum_{i=1}^{n} 1[h(x)\neq y]}{n}
$$
\begin{definition}[(Agnostic) PAC Learnability]
A hypothesis class $\mathcal{H}$ is agnostic PAC learnable if there exist a function (sample complexity) $m_\mathcal{H}:(0,1)^2 \to \mathbb{N}$ and a learning algorithm with the following property: For every $\epsilon, \delta \in (0,1)$ and for every distribution $\mathcal{D}$ over $\mathcal{X} \times \mathcal{Y}$, when running the learning algorithm on $m \geq m_\mathcal{H}(\epsilon,\delta)$ i.i.d examples generated by $\mathcal{D}$, the algorithm returnes a hypothesis $h$ such that, with probability of at least $1 - \delta$ (over the choice of m training examples), 
$$
L_\mathcal{D}(h) = \min_{h'  \in \mathcal{H}} L_\mathcal{D}(h') +\epsilon.
$$
\end{definition}
This is in contrast to realizability case of PAC learning, in which it is mandatory for the learner to achieve a small error in absolute terms and not relative to the best error attainable by the hypothesis class. Agnostic PAC learning generalizes the notion of PAC learning \citep{Haussler 1992}
\begin{corollary}\label{hoe}
let $\mathcal{H}$ be a finite hypothesis class, let $Z$ be a domain, and let $\ell : \mathcal{H} \times Z \to [0,1]$ be a loss function. Then, $\mathcal{H}$ enjoys the uniform convergence property with the sample complexity 
\begin{equation}
    m_\mathcal{H}^{UC}(\epsilon, \delta) \leq \frac{log(2|{\mathcal{H}|/\delta})}{2\epsilon^2}
\end{equation}
Also, the class is Agnostic PAC leranble by the ERM algorithm with sample complexity 
\begin{equation}
    m_\mathcal{H}^{UC}(\epsilon, \delta) \leq   m_\mathcal{H}^{UC}(\epsilon/2, \delta) \leq \frac{2log(2|{\mathcal{H}|/\delta})}{\epsilon^2}
\end{equation}
\end{corollary}
\begin{definition}[Nonuniform Learnability] \label{pythagorean1}
A hypothesis class $\mathcal{H}$ is nonuniform learnable if $\exists $ a learning algorithm $A$ and a function $ m_\mathcal{H}^{NUL} : (0,1)^2 \times \mathcal{H} \to \mathbb{N}$  such that, for every $(\epsilon,\delta) \in (0,1)$ and for every $h \in \mathcal{H}$, if $m \geq$ $m_\mathcal{H}^{UC}(\epsilon ,\delta ,h)$ then for every distribution $\mathcal{D}$, with probability of at least $1-\delta$ over the choice of $S \sim \mathcal{D}^m$, it holds that
$$
    L_\mathcal{D}(A(S)) \leq L_\mathcal{D}(h) +\epsilon.
$$
\end{definition}
In  agnostic PAC learnable and nonuniform learnability, we know that the hypothesis will be $(\epsilon, \delta)$- competitve with every other hypothesis in that particular class with key difference being the sample size $m$ may depend on the hypothesis $h$ to which the error of $A(S)$ is compared in nonuniform learnable case.
\begin{theorem}\label{non}
A hypothesis class of $\mathcal{H}$ which is a countable union of hypothesis classes, $\mathcal{H}= \bigcup\limits_{n=1}^{\mathbb{N}} \mathcal{H}_{n}$, where each $\mathcal{H}_{n}$ has the uniform convergence property. Then, $\mathcal{H}$ is nonuniform learnable. 
\end{theorem}
Also, from the fundamental theorem of learning \citep{Vapnik and Chervonenkis 1974} we can also state that each $\mathcal{H}_n$ enjoys also the agnostic PAC learnability.
\begin{definition}\label{d5}
Let $\mathcal{K}$ be a probabilistic model, such that each $\mathcal{P} \in \mathcal{K}$ is
a probability distribution. Roughly, a statistical procedure is called consistent relative
to $\mathcal{K}$ if, for all $\mathcal{P}* \in \mathcal{K}$, the following holds: suppose data are distributed according to $\mathcal{P}*$. Then given enough data, the learning method will learn a good approximation of $\mathcal{P}*$ with high probability.
\end{definition}
\section{Setup}\label{Section3}
\subsection{General Background}
\textbf{Language:} If we wish to describe a hypothesis class $\mathcal{H}$. let $\Sigma = \{0,1\}$. we denote $\sigma$ as string which is a finite sequence of symbols from $\Sigma$. Also, $|\sigma|$ denote the length of the string and the set of all length of the strings is denoted by $\Sigma^*$. A description language for $\mathcal{H}$ is a function $d: \mathcal{H} \to \Sigma^*$ \par
Let us assume that $\mathcal{H}_{all}$ represents a class containing all function or hypothesis that goes from $\mathcal{X} \to \{0,1\}$. We define a description language $d$,
\begin{equation}\label{hoe2}
    d: \mathcal{H}_{all} \to \{0,1\}^\mathbb{N}
\end{equation}
Let $\mathcal{H}_n$ be the collection of function $h \in \mathcal{H}_{all}$ that can be described by the first $n$ bits. More formally, the elements of $\mathcal{H}_{n}$ can be described as follows : $h \in \mathcal{H}_n$ if and only if there exist some $h^{(n)}$ $\subseteq$ $\{0,1\}^n$ such that $h=\{h \in \mathcal{H}_n|\{0,1\}^n \in h^{(n)}\}$.
\begin{remark}
We also know that $\mathcal{H}_{all}$ has a strictly bigger cardinality than $\mathcal{H}= \bigcup\limits_{n=1}^{\mathbb{N}} \mathcal{H}_{n}$, where each $\mathcal{H}_n$ has finite VC dimension and each  $h \in \mathcal{H}_n$ can be described by first $n$ bits or the combination of the first $n$ bits.
\end{remark}
\begin{theorem}\label{t7}
$\mathcal{H}$ is an algebra and not a $\sigma$-algebra
\end{theorem}
\begin{proof}
From the definition of an algebra we can say that $\mathcal{H}$ is an algebra.\par
Now, consider the following example let $h_x \in \mathcal{H}_{all}$ be a hypothesis that can be described in $\{0,1\}^\mathbb{N}$ where the bit 1 takes the odd positions. Clearly, $h_x \notin \mathcal{H}$ since we cannot describe $h_x$ in finite n bits. On the other hand, $h_x$ can be expressed as a countable intersections of elements in $\mathcal{H}$ :
$$
    h_x=\bigcap\limits_{n=1}^{\mathbb{N}} h_{2n-1}
$$
and hence, $\mathcal{H}$ is not a $\sigma$-algebra
\end{proof}
\begin{theorem}[Caratheodory’s Extension Theorem]\label{pythagorean}
Let $\mathcal{J}$ be a algebra of subsets of $\Omega$ (sample space). Let $\mathbb{P}_0: \mathcal{J} \to [0,1]$ with $\mathbb{P}_0(\phi)=0$ and $\mathbb{P}_0(\Omega)=1$, satisfying the finite superadditvity property and also the countable monotonicity property  
\begin{equation}
\mathbb{P}_0(\bigcup_{i=1}^kA_i) \geq \mathlarger{\sum_{i=1}^k\mathbb{P}_0(A_i)}
\end{equation}
whenever $A_1,\hdots,A_k \in \mathcal{J}$, and $\bigcup_{i=1}^kA_i \in \mathcal{J}$ and the $\{A_i\}$ are disjoint. Then there is a $\sigma$-algebra $\mathcal{M} \supseteq \mathcal{J}$, and a countably additive probability measure $\mathbb{P}$ on  $\mathcal{M}$, such that $\mathbb{P}(A)=\mathbb{P}_0(A)$ for all $A \in  \mathcal{J}$ i.e. $(\Omega,\mathcal{M},\mathbb{P})$ is a valid probability triple, which agrees with our previous probabilities on $\mathcal{J}$.
\end{theorem}
 Now, consider the smallest $\sigma$-algebra containing all the elements of $\mathcal{H}$, define $\mathcal{H}_\sigma=\sigma(\mathcal{H})$. Consider the probability triple $(\mathcal{H}_{all},\mathcal{H},\mathbb{P}_0)$. Here, $\mathcal{H}$ is an algebra so by defining a finitely additive function $\mathbb{P}_0$ on $\mathcal{H}$ that also satisfies $\mathbb{P}(\Omega)=1$, for example, $\mathbb{P}_0(A)=\frac{|A^{(n)}|}{2^n}$. Then, we shall subsequently extend $\mathbb{P}_0$ to a probability measure $\mathbb{P}$ on $\mathcal{H}_\sigma$ i.e. $(\mathcal{H}_{all},\mathcal{H}_\sigma,\mathbb{P})$ via Extension Theorem (see Theorem \ref{pythagorean}).\par
Now each $h \in \mathcal{H}$ has a proper measure also $\mathcal{H}$ is nonuniform learnable (from Theorem \ref{pythagorean1}) and each $\mathcal{H}_n$ corresponds to having $2^{2^n}$ elements or hypothesis. Now, model selection using SRM where the rule applies a "bound minimization" approach \citep{Shawrtz and Ben David} in which the overall goal of the paradigm is to find a hypothesis that minimizes a certain upper bound on the true error(risk). The bound that the SRM rule wishes to minimize is, where  $w: \mathbb{N} \rightarrow[0,1]$ be a function such that $\sum_{n=1}^{\infty} w(n) \leq 1 .$ \par 
\begin{equation}\label{eq10}
  \left|L_{\mathcal{D}}(h)-L_{S}(h)\right| \leq \epsilon_{n}(m, w(n) \cdot \delta)
\end{equation}
\begin{equation}\label{eq11}
    \epsilon_{n}(m, \delta)=\min \left\{\epsilon \in(0,1): m_{\mathcal{H}_{n}}^{\mathrm{UC}}(\epsilon, \delta) \leq m\right\} 
\end{equation}

Furthermore, the VC dimension of union of finite classes  $\mathcal{H}_1, \ldots, \mathcal{H}{r}$ (hypothesis classes over some fixed domain set $\mathcal{X}) (Shawrtz and Ben David)$. This case is not for countably infinite union of hypothesis classes. Let $d=\max_{n} \operatorname{VC dim}\left(\mathcal{H}_n\right)$ 
\begin{equation}
    \operatorname{VC dim}\left(\cup_{n=1}^{r} \mathcal{H}_{i}\right) \leq 4 d \log (2 d)+2 \log (r)
\end{equation}
\subsection{Generalize Measure Learnability}
As we are trying to build around SRM, measurable class $\mathcal{H}$ follows the bound on true error (risk) and empirical error (risk) given in the following theorem
\begin{theorem}\label{t9}

Let $w: \mathbb{N} \rightarrow[0,1]$ be a function such that $\sum_{n=1}^{\infty} w(n) \leq 1 .$ Let $\mathcal{H}$ be a hypothesis class that can be written as $\mathcal{H}=\bigcup_{n \in \mathbb{N}} \mathcal{H}{n},$ where for each $n$ $\mathcal{H}{n}$ satisfies the uniform convergence property with a sample complexity function $m_{\mathcal{H}{n}}^{U C} .$Then, for every $\delta \in(0,1)$ and distribution $\mathcal{D},$ with probability of at least $1-\delta$ over the choice of $S \sim \mathcal{D}^{m},$ the following bound holds (simultaneously) for every $n \in \mathbb{N}$ and $h \in \mathcal{H}_{n}$.
$$
\left|L_{\mathcal{D}}(h)-L_{S}(h)\right| \leq \sqrt{\frac{-\log (w(n))+\log (2^{2^n+1} / \delta)}{2 m}}
$$
Therefore, for every $\delta \in(0,1)$ and distribution $\mathcal{D},$ with probability of at least $1-\delta$ it holds that for all $h \in \mathcal{H}$
\begin{equation}
   \left|L_{\mathcal{D}}(h)-L_{S}(h)\right| \leq  \underset{h \in \mathcal{H}_n}{\operatorname{min}}\left[\sqrt{\frac{-\log (w(n))+\log (2^{2^n+1}  / \delta)}{2 m}}\right]
\end{equation}

\end{theorem}
\begin{proof}
Let $\mathcal{H}$ be a countable hypothesis class. Then, we can write $\mathcal{H}$ as a countable union of each individual class $\mathcal{H}_n$ , namely, $\mathcal{H}=\bigcup_{n \in \mathbb{N}}\mathcal{H}_{n}$. As the number of elements of a particular hypothesis class $\mathcal{H}_n$ is $2^{2^n}$ so by corollary (\ref{hoe}), each class has the uniform convergence property with rate 
$$
 m^{\mathrm{UC}}(\epsilon, \delta)=\frac{\log (2^{2^n+1} / \delta)}{2 \epsilon^{2}}    
$$
Therefore, the function $\epsilon_{n}$ given in Equation (\ref{eq11})
becomes
$$
 \epsilon_{n}(m, \delta)=\sqrt{\frac{\log (2^{2^n+1}/ \delta)}{2 m}}
$$
and the SRM rule becomes accroding to eqaution (\ref{eq10})
\begin{equation}\label{hoe1}
    \underset{h_{n} \in \mathcal{H}}{\operatorname{argmin}}\left[L_{S}(h)+\sqrt{\frac{-\log (w(n))+\log (2^{2^n+1} / \delta)}{2 m}}\right]
\end{equation}

Therefore, the bound becomes
$$
\left|L_{\mathcal{D}}(h)-L_{S}(h)\right| \leq  \underset{h \in \mathcal{H}_n}{\operatorname{min}}\left[\sqrt{\frac{-\log (w(n))+\log (2^{2^n+1}  / \delta)}{2 m}}\right]
$$
\end{proof}
With different weighting functions such as $w(n)=\frac{6}{\pi^2n^2}$ and $w(n)=2^{-n}$ equation (\ref{hoe1}) becomes simultaneously,
$$
  \underset{h_{n} \in \mathcal{H}}{\operatorname{argmin}}\left[L_{S}(h)+\sqrt{\frac{\log (\frac{\pi^2n^2}{6})+\log (2^{2^n+1} / \delta)}{2 m}}\right]
$$
and 
$$
\underset{h_{n} \in \mathcal{H}}{\operatorname{argmin}}\left[L_{S}(h)+\sqrt{\frac{\log (2^n)+\log (2^{2^n+1} / \delta)}{2 m}}\right]
$$
In particular, it suggests trading off empirical risk for saving the size of the class. This result suggests a learning paradigm for $\mathcal{H}$ that given a training set $S$ it searches for $h \in \mathcal{H}$ that minimizes the bound given in equation (\ref{hoe1}), as formalized in the following pseudo-code:\\~\\
\SetKwInput{KwData}{The data}
\SetKwInput{Kwinput}{The result}
\framebox{
\begin{algorithm}[H]
 \KwData{\begin{center}\begin{tabular}{ c c c }
 $\mathcal{H}$ is a countable hypothesis class \\  
 $\mathcal{H}_{all}$ is described by a language over $\{0,1\}^\mathbb{N}$ ( equation \ref{hoe2}) \\  
\end{tabular}
\end{center}}
 \Kwinput{ \textbf{Input:} We fix the range of $n$ ( maybe a fixed value or a range) \\ A training set $S \sim \mathcal{D}^m$, confidence $\delta$ \\
 \textbf{Output:$ h \in {\operatorname{argmin}}\left[L_{S}(h)+\sqrt{\frac{-\log (w(n))+\log (2^{2^n+1} / \delta)}{2 m}}\right]$}}
\caption{\\~\\ GML Algorithm}
\end{algorithm}
}

\subsection{Sample and Computational Complexity Bounds}
As it is mentioned in (Shawrtz and Ben David) that the gap between the sample complexity of a nonuniform learner $\mathcal{H}$ ad specific $\mathcal{H}_n$, where the VC dimension of each $\mathcal{H}_n$ is $n$, can be applied here as well
$$
m_{\mathcal{H}}^{\mathrm{NUL}}(\epsilon, \delta, h)-m_{\mathcal{H}_{n}}^{\mathrm{UC}}(\epsilon / 2, \delta) \leq 4 C \frac{2 \log (2 n)}{\epsilon^{2}}
$$
The above equation suggests that the cost of the relation of prior knowledge of $\mathcal{H}_n$ to the union of $\mathcal{H}_n$ that is $\mathcal{H}$ depends on the log of the index of the class in which $h$ resides. Moreover, cost increases with the index of the class, which can also be view as a good priority order among $\mathcal{H}_n$ on the hypothesis in $h$\par
For example, $\mathcal{H}$ can be the set of all predictors that can be implemented by a C++ program at most 10000 bits of code. the sample complexity has a dependence on the size of $\mathcal{H}$. i.e. $c(10000+log(c/\delta))/\epsilon^c$. but the number of hypotheses is $2^{10000}$. Also, if $\mathcal{H}_n$ is the set of functions that can be implemented by a C++ program written in at most $n$ bits of code, the run time grows exponentially with $n$, which also implies that the exhaustive search approach is not very practical. In fact for many cases, for individual $\mathcal{H}_n$ it is NP-hard to decide if there is a 2-term DNF that correctly classifies all examples in a training sample (Pitt and Valiant 1988). Two-layer linear threshold networks are also NP-hard (Blum and Rivest 1992) see also (Dasgupta et. al). Maximizing agreements with monotone monomials is NP-hard (Angluin and Laird 1988). A similar result for general monomials (Kearns and Li 1993), and also for half-spaces (Hoffgen and Simon 1995). Unless $P=NP$, but there is no algorithm whose running time is polynomial that is guaranteed to find an ERM (Empirical Risk Minimization) hypothesis for these problems (Ben-David et. al 2003). So the important thing to focus on is that for any of the hypothesis 
classes (which are NP-hard to learn) that we discuss there are input distributions that are efficiently learnable using some of the other hypothesis classes. This gives an important idea in model selection. As in this paper model selection has been dealt with information complexity. The choice of model can influence the computational complexity of finding a good hypothesis, that a poor choice can simultaneously influence the algorithm both by failing to explain the data and by making it hard to find a good hypothesis (Ben-David et. al 2003). Another implication to model selection is that when the learner sees the training data, it is advantageous to be able to the hypothesis class relative to which learning algorithm is proceeding. But the task of choosing a class to work with for a given data is in itself computationally difficult. (Ben-David et. al 2003). So to get around this hardness of learning we have to restrict the learning algorithm to search within a particular index or a range of index like in GML so limiting the hypothesis class is our only way.
\section{Discussions}
New paradigms such as deep belief networks (Bengio 2009) to the theoretical nature of universal kernels and universal priors (Hutter 2011) and MDL learning for universal coding (Schmidt and Lipson 2009). These paradigms offers a notion of universal learner that can be applied to any learning task without the need for prior knowledge or weaker inductive bias. According to No-free-lunch theorem leaner cannot be guaranteed to be able to learn every labeling rule over some domain,
unless it has access to a labeled training sample whose size is comparable to the size of the underlying domain. So, as per theorem \ref{non}, we can say that $\mathcal{H}$ can be learnable only if it is countable of specific $\mathcal{H}_n$ (finite VC dimension). For any description language, in this paper, the computability is not considered. If the learner is non-computable then that can be useless in practical application. Learnability by computable learners where the notion of CPAC learnability is introduced where there exist classes that are learnable by this setting but not learnable by the computable learners. (Agarwal et. al 2020) and as established by (Soloveichik 2008) no computable learner exists. Although, by relaxing the training size requirements we can comment on the computability of universal learners (Goldeich and Ron 1996).\par As we try to put the measure on only the classes which have finite VC dimension. Classes that do not follow the norm as given in Theorem \ref{non} have an undefined measure. Each singleton element represents a hypothesis whereas the rest of the elements might indicate other hypotheses or a special group of hypotheses. There might be a uniform measure between hypothesis or among the group of hypotheses that will be well suited but not discussed in this paper. For this current paper, our scope is limited to the mathematical aspects and the computational aspect has not been dealt with  in this paper. \par
In our discussions involving Occam's Razor, we will say different codes would have led to different description lengths, and thus, to different models. By simply changing the encoding method, we are able to make `complex' things `simple' and contrariwise. Also, whether or not the true data generating the mechanism is complex, it should be an honest strategy to prefer simple models for small sample sizes. We are now in a very position to administer one formalization of this informal claim: it's simply the very fact that GML with their already established preference for `simple' models with small parametric complexity, are typically statistically consistent achieving good rates of convergence ( definition \ref{d5}), whereas methods like maximum likelihood which do not take model complexity into consideration are typically in-consistent whenever they're applied to complex enough models like the set of polynomials of every degree or the set of Markov chains of all orders. This has implications for the standard of predictions: with complex enough models, irrespective of what percentage training data we observe, if we use the maximum likelihood estimation to predict future data from the identical source, the prediction error won't converge to the prediction error that might be obtained if the true distribution were known, consistency isn't the sole desirable property of a learning method, and it may be that in some particular settings, and under some particular performance measures, some alternatives will outperform GML. Yet it remains the case that every one methods that successfully deal with models of arbitrary complexity have a built-in preference for choosing simpler models at small sample sizes MDL, penalized minimum error estimators (Barron 1999) and also the Akaike criterion (Burnham and Anderson 2002) the result invariably being that during this way, good convergence properties may be obtained. While these approaches measure complexity in an different manner from GML and attach different relative weights on the information and complexity, the basic idea of finding a trade-off between error and complexity remains same.

\section{Conclusion and Future Work}
In this paper, we have initiated an investigation of analyzing the statistical notions of nonuniform learnability using measure theory. We have shown that following the pre-existing theorems of model selection we can still achieve a statistically consistent algorithm (GML). As we are limited by the computational aspect of the runtime aspect of the algorithm we try to restrict the hypotheses space. Although, making GML a union of computable learners and making the whole algorithm computable so that model selection can be feasible can be considered as future work. The GML is an explication to Occam’s Razor.
\\ 
\section{References}
[Vapnik and Chervonenkis 1974] Vapnik, V. N. and Chervonenkis, A. Y 1974, Theory of pattern recognition, Nauka,
Moscow. (In Russian).\newline\bigskip
[Vapnik 1995] Vapnik, V. .1995, The Nature of Statistical Learning Theory, Springer.\newline\bigskip
[Rissanen 1978]  Rissanen, J. 1978, `Modeling by shortest data description', Automatica 14, 465-471.\newline\bigskip
[Rissanen 1983]  Rissanen, J. 1983, `A universal prior for integers and estimation by minimum description
length', The Annals of Statistics 11(2), 416-431.\newline\bigskip
[Vovk and Gammerman 1999] V. Vovk and A. Gammerman. Complexity approximation principle. The Computer
Journal, 42(4):318–322, 1999.\newline\bigskip
[Blumer Ehrenfraut Hausslerand Warmuth 1987] Blumer, A., Ehrenfeucht, A., Haussler, D.  Warmuth, M. K. (1987), `Occam's razor',
Information Processing Letters 24(6), 377-380.\newline\bigskip
[Tikhonov 1943] Tikhonov, A. N. (1943), `On the stability of inverse problems', Dolk. Akad. Nauk SSSR
39(5), 195-198.\newline\bigskip
[Haussler 1992]  David Haussler. Decision theoretic generalizations of the PAC model for neural
net and other learning applications. Inf. Comput., 100(1):78-150, 1992. doi: 10.1016/0890-5401(92)90010-D.\newline\bigskip
[Shawrtz and Ben David]  Shwartz and Shai Ben-David. Understanding Machine Learning: From Theory to Algorithms. Cambridge University Press, 2014.
39(5), 195-198.\newline\bigskip
[Pitt and Valiant 1988]  L. Pitt, L.G. Valiant, Computational limitations on learning from examples, J. Assoc. Comput. Mach. 35 (4) (1988) 965–984.
39(5), 195-198.\newline\bigskip
[Blum and Rivest 1992]  A.L. Blum, R.L. Rivest, Training a 3-node neural network is NP-complete, Neural Networks 5 (1) (1992) 117–127.
39(5), 195-198.\newline\bigskip
[Dasgupta et. al]  B. DasGupta, H.T. Siegelmann, Eduardo D. Sontag, On the complexity of training neural networks with continuous activation functions, IEEE Trans. Neural Networks 6 (6) (1995) 1490–1504.\newline\bigskip
[Angluin and Laird 1988]  D. Angluin, P.D. Laird, Learning from noisy examples, Mach. Learning 2 (1988) 343–370.\newline\bigskip
[Kearns and Li 1993]  M. Kearns, M. Li, Learning in the presence of malicious errors, SIAM J. Comput. 22 (4) (1993) 807–837.\newline\bigskip
[Hoffgen and Simon 1995]  K.U. Hoffgen, H.U. Simon, K.S. Van Horn, Robust trainability of single neurons, J. Comput. System Sci. 50 (1)
(1995) 114–125.\newline\bigskip
[Ben-David et. al 2003]  Shai Ben-David, Nadav Eiron,and Philip M. Long,  Journal of Computer and System Sciences 66 (2003) 496–514\newline\bigskip
[Bengio 2009]  Yoshua Bengio. Learning deep architectures for ai. Foundations and Trends in Machine Learning,
2(1):1–127, 2009.\newline\bigskip
[Hutter 2011]  Marcus Hutter. Universal learning theory. CoRR, abs/1102.2467, 2011.\newline
[Schmidt and Lipson 2009] Michael Schmidt and Hod Lipson. Distilling free-form natural laws from experimental data.Science 3, 324 no. 5923 pp. 81-85, 2009.\newline\bigskip
[Agarwal et. al 2020]  2020 S. Agarwal, N. Ananthakrishnan, S. Ben-David, T. Lechner  R. Urner. Proceedings of Machine Learning Research vol 117:1–13,\newline\bigskip
[Soloveichik 2008]  David Soloveichik. Statistical learning of arbitrary computable classifiers. CoRR,abs/0806.3537, 2008.\newline\bigskip
[Goldeich and Ron 1996]  Oded Goldreich and Dana Ron. On universal learning algorithms. Online version, 1996.\newline\bigskip
[Barron 1999] Andrew Barron, Lucien Birgue, Pascal Massart  Risk bounds for model selection via penalization Probab. Theory Relat. Fields 113, 301–413 (1999)\newline\bigskip
[Burnham and Anderson 2002] Kenneth P. Burnham David R. Anderson
Model Selection and Multimodel Inference

\vskip 0.2in


\newpage

\appendix

\end{document}